\documentclass[wcp]{jmlr}

\usepackage{dsfont}

\newtheorem{hypothesis}{Hypothesis}

\usepackage{longtable}

\usepackage{booktabs}

\usepackage{lineno}
\usepackage[export]{adjustbox}

\DeclareMathOperator*{\E}{\mathbb{E}}
\DeclareMathOperator*{\Var}{\mathrm{Var}}
\DeclareMathOperator*{\Cov}{\mathrm{Cov}}
\DeclareMathOperator*{\Corr}{\mathrm{Corr}}

\makeatletter
\let\Ginclude@graphics\@org@Ginclude@graphics 
\makeatother

\jmlrvolume{222}
\jmlryear{2023}
\jmlrworkshop{ACML 2023}

\title[Intractability of Learning the Discrete Logarithm]{Intractability of Learning the Discrete Logarithm with Gradient-Based Methods}

  \author{\Name{Rustem Takhanov} \Email{rustem.takhanov@nu.edu.kz}\\
   \Name{Maxat Tezekbayev} \Email{maxat.tezekbayev@nu.edu.kz}\\
   \Name{Artur Pak} \Email{artur.pak@nu.edu.kz}\\
   \addr{Department of Mathematics, Nazarbayev University, Astana, Kazakhstan}
   \AND
   \Name{Arman Bolatov} \Email{arman.bolatov@nu.edu.kz}\\
   \addr{Department of Computer Science, Nazarbayev University, Astana, Kazakhstan}
   \AND
   \Name{Zhibek Kadyrsizova} \Email{zhibek.kadyrsizova@nu.edu.kz}\\
   \addr{Department of Mathematics, Nazarbayev University, Astana, Kazakhstan}
   \AND
   \Name{Zhenisbek Assylbekov} \Email{zassylbe@pfw.edu}\\
 \addr Department of Mathematical Sciences, Purdue University Fort Wayne, Fort Wayne, IN, USA}

\editors{Berrin Yan{\i}ko\u{g}lu and Wray Buntine}

\begin{document}

\maketitle

\begin{abstract}
The discrete logarithm problem is a fundamental challenge in number theory with significant implications for cryptographic protocols. In this paper, we investigate the limitations of gradient-based methods for learning the parity bit of the discrete logarithm in finite cyclic groups of prime order. Our main result, supported by theoretical analysis and empirical verification, reveals the concentration of the gradient of the loss function around a fixed point, independent of the logarithm's base used. This concentration property leads to a restricted ability to learn the parity bit efficiently using gradient-based methods, irrespective of the complexity of the network architecture being trained.

Our proof relies on Boas-Bellman inequality in inner product spaces and it involves establishing approximate orthogonality of discrete logarithm's parity bit functions through the spectral norm of certain matrices. Empirical experiments using a neural network-based approach further verify the limitations of gradient-based learning, demonstrating the decreasing success rate in predicting the parity bit as the  group order increases.
\end{abstract}
\begin{keywords}
Discrete Logarithm, Gradient-based Learning, Cryptographic Protocols.
\end{keywords}

\section{Introduction}

Today, artificial intelligence is able to solve problems that seemed extremely difficult for machines 10 years ago. The most famous success stories include the victory of the machine over a professional Go player \citep{silver2016mastering}; prediction of the spatial structure of a protein with high accuracy \citep{jumper2021highly}; a chatbot capable of working in a conversational mode, supporting requests in natural languages \citep{chatgpt}. Since all these stories are based on deep neural networks trained by gradient-based methods, it may seem that at this pace there will soon be no problems that would be beyond the capacity of gradient-based learning. 

However, \cite{DBLP:conf/icml/Shalev-ShwartzS17} showed the failure of gradient-based methods to learn some rather simple functions such as the parity function. In this paper, we give an example of another simple function that gradient-based methods provably cannot learn. This function is a parity bit of the \textbf{discrete logarithm} in the additive group of integers modulo $p$. It is known that with the help of the extended Euclidean algorithm, the discrete logarithm in this group can be computed in $O(n)$ operations, where $n$ is the bitlength of $p$ (see Section~\ref{sec:prelim} for details). So there is a Boolean circuit with complexity $O(n^2)$, which implements a parity bit of the discrete logarithm (with a fixed base). Since each logic gate can be implemented by a small number of neurons and weights, this circuit can be converted into  a compact neural network with $\mathrm{poly}(n)$ parameters. However, we prove formally that gradient-based methods cannot efficiently train such a network. 

In fact, we have obtained a more general result (Theorem~\ref{thm:main}) which says that when trying to learn the parity bit of the discrete logarithm in \emph{any} finite cyclic group of prime order, and not only in $(\mathbb{Z}_p,+)$, the gradient carries negligible information about the target function. It has long been known in cryptography that the discrete logarithm problem (DLP) in a carefully chosen cyclic group (for example, in the group of points on an elliptic curve \citep{DBLP:conf/crypto/Miller85}) is hard in the sense that at the moment there is no $\mathrm{poly}(n)$ algorithm for solving the DLP in general case. The intractability of the DLP in such groups forms the basis for various cryptographic protocols, including public-key encryption, digital signatures \citep{DBLP:journals/tit/Elgamal85}, and key exchange \citep{DBLP:journals/tit/DiffieH76}. From this point of view, our result is the \textbf{provable security of DLP-based cryptosystems against gradient-based attacks}.

\subsection{Related Work}
The main source of inspiration for us is the work of \cite{DBLP:conf/icml/Shalev-ShwartzS17}, which, among other things, shows the intractability of learning a class of orthogonal functions using gradient-based methods. We emphasize that their result is not directly applicable to the class of functions that we consider in this paper (the parity bit of the discrete logarithm), since these functions are not orthogonal with respect to a uniform distribution over the domain. However, they are \emph{approximately} pairwise orthogonal, and the proof of this fact is the core of our work (Section~\ref{sec:proofs}). In addition, our adaptation of the proof method by \cite{DBLP:conf/icml/Shalev-ShwartzS17} using the Boas-Bellman inequality (Section~\ref{sec:proof_main}) deserves special attention, as it allows us to extend the failure of gradient-based learning  to a wider class of approximately orthogonal functions.

It should be noted that the relationship between orthogonal functions and hardness of learning is not new and has been established in the context of statistical query (SQ) learning model of \cite{DBLP:conf/stoc/Kearns93}. Moreover, this relationship was characterized by \cite{DBLP:conf/stoc/BlumFJKMR94} in terms of the statistical dimension of the function class, which essentially corresponds to the largest possible set of functions in the class which are all \emph{approximately} pairwise orthogonal. The hardness of learning a class of boolean functions in an SQ model is usually proven through a lower bound on the statistical dimension of the class. 
It is noteworthy that gradient-based learning with an approximate gradient oracle can be implemented through the SQ algorithm \citep{DBLP:conf/soda/FeldmanGV17}, which means that our result on the approximate orthogonality of the considered class of functions (Lemma~\ref{lem:sum_of_squares}) immediately gives the hardness of learning this class with gradient-based methods. Nevertheless, we believe that the proof of this result directly (without resorting to the SQ proxy) deserves attention, since it allows us to establish that the low information content of the gradient is the very reason why gradient learning fails.

Theorem~1 of \citet{liu2021rigorous} states that assuming the classical hardness of the DLP, no efficient classical algorithm can learn the concept class constructed by the authors. Thus their result, although applicable to all classical learning algorithms, is conditioned by a strong assumption. In our paper, we show the \emph{unconditional} hardness of learning the parity bit of a discrete logarithm by any gradient-based method (e.g., SGD, RMSProp, Adam, etc.), i.e. we do not make any assumptions on the hardness of the DLP itself.

\subsection{Notation}
Bold-faced lowercase letters ($\mathbf{x}$) denote vectors, bold-faced uppercase letters ($\mathbf{A}$) denote matrices. Regular lowercase letters ($x$) denote scalars (or set elements), and regular uppercase letters ($X$) denote random variables (or random elements). $\|\cdot\|$ denotes the Euclidean norm: $\|\mathbf{x}\|:=\sqrt{\mathbf{x}^\top\mathbf{x}}$. For $\mathbf{x}\in\mathbb{C}^n$, conjugate transpose is denoted by $\mathbf{x}^\dag$. For any finite set $\mathcal{S}$, sampling $X$ uniformly from $\mathcal{S}$ is denoted by $X\sim\mathcal{S}$. For two functions $f, g$ on a finite set $\mathcal{S}$, let $\langle f,g\rangle:=\E_{X\sim\mathcal{S}}[f({X})\cdot g({X})]$ and $\|f\|:=\sqrt{\langle f,f\rangle}$. For a matrix $\mathbf{A}\in\mathbb{R}^{m\times n}$, its spectral norm is denoted by $\|\mathbf{A}\|$.

We use Vinogradov notation, i.e. given $f:\,\mathbb{R}\to\mathbb{R}$ and $g:\,\mathbb{R}\to\mathbb{R}_+$, we write $f\ll g$  if there exist $x_0,\alpha\in\mathbb{R}_+$ such that for all $x>x_0$ we have $|f(x)|\le\alpha g(x)$. When $f:\mathbb{R}\to\mathbb{R}_+$, we write $f\asymp g$  if $f\ll g$ and $g\ll f$. For $x>0$, we write $f=\mathrm{poly}(x)$ if there exists $k\in\mathbb{N}$ such that $f\ll x^k$. We write $f=\widetilde{O}(g)$ if $f\ll g\ln^k g$ for some $k>0$. Similarly, $f=\widetilde\Omega(g)$ means that $g\ln^k g\ll f$ for some $k>0$.

$\mathbb{Z}_p$ is the set $\{0,1,\ldots, p-1\}$, equipped with two operations, $+$ and $\times$, which work as usual addition and multiplication, except that the results are reduced modulo $p$.
 $\mathbb{Z}_p^\ast$ denotes the set of elements in $\mathbb{Z}_p$ that are relatively prime to $p$. We are mainly interested in the case when $p$ is a prime number greater than 2. In this case $\mathbb{Z}_p^\ast=\{1,\ldots,p-1\}$. By abuse of notation, we sometimes treat elements of $\mathbb{Z}_p$ (and of $\mathbb{Z}_p^\ast$) as integers in $\mathbb{Z}$. Given two positive integers $a$ and $p$, $a \bmod p$ is the remainder of the Euclidean division of $a$ by $p$, where $a$ is the dividend and $p$ is the divisor. 

\section{The Discrete Logarithm Problem}\label{sec:prelim}


Let $(\mathcal{G}, \circ)$ be a finite group,  $a\in\mathcal{G}$ an element of order $p$, and  $x\in\langle a\rangle$, where $\langle a\rangle:=\{\underbrace{a\circ a\circ\ldots\circ a}_{k\text{ times}}:\,0\le k\le p-1\}$ is the cyclic group generated by $a$. The \textbf{discrete logarithm problem (DLP)} is finding the integer $k$, $0\le k\le p-1$, such that
    $$
    \underbrace{a\circ a\circ\ldots\circ a}_{k\text{ times}}=x.
    $$
    This integer $k$ is called the \emph{discrete logarithm} of $x$ to the base $a$, and we will denote it by $\log_a x$.

It is important to understand that there are finite groups in which the DLP is not hard (computationally). As an example, consider the additive group of integers modulo prime. For example, if we take $p=11$, $(Z_{11}, +)$ is a finite cyclic group in which every non-zero element is primitive.\footnote{An element $a$ of a cyclic group $(\mathcal{G},\circ)$ is called \emph{primitive} if every element $x\in\mathcal{G}$ can be written as $x=\underbrace{a\circ a\circ \ldots \circ a}_{k\text{ times}}$ for some $k$, i.e. $a$ generates the entire group.} Here, for example, is how the element $a=2$ generates the entire group:
\begin{table}[htbp]
    \centering
    \begin{tabular}{c | c c c c c c c c c c c}
    \toprule
         $k$ & 0 & 1 & 2 & 3 & 4 & 5 & 6 & 7 & 8 & 9 & 10 \\
    \midrule
         $k\cdot a$ & 0 & 2 & 4 & 6 & 8 & 10 & 1 & 3 & 5 & 7 & 9 \\
    \bottomrule
    \end{tabular}
    \caption{Generating the group $(\mathbb{Z}_{11},+)$ from $a=2$.}
    \label{tab:group}
\end{table}

Suppose we want to solve the DLP for the element $x=3$, that is, we want to find an integer $k$ such that
$$
\underbrace{2+2+\ldots+2}_{k\text{ times}}\equiv 3\bmod 11.
$$
This can be done as follows. Even though the group operation is addition, we can express the relationship between $x$, $k$, and the discrete logarithm using multiplication:
\begin{equation}
k\cdot 2\equiv 3\bmod 11\label{eq:to_solve}
\end{equation}
To solve the equation \eqref{eq:to_solve} for $k$, we just need to find the (multiplicative) inverse for $a = 2$:
$$
    k \equiv 2^{-1}\cdot 3\bmod 11.
$$
Using, for example, the extended Euclidean algorithm, we can compute $2^{-1}\equiv6\bmod 11$ and so the value of the discrete logarithm is
$$
k\equiv 2^{-1}\cdot3\equiv 7\bmod 11.
$$
Table \ref{tab:group} confirms the correctness of the found value.

The above technique can be used for any group $(\mathbb{Z}_p, +)$ and any non-zero elements $a, x\in\mathbb{Z}_p$. Accordingly, the DLP is a computationally easy problem over $(\mathbb{Z}_p, +)$. Such groups cannot be used for cryptography. However, our subsequent analysis will show that even in them, learning just a single bit of the discrete logarithm is intractable for gradient-based methods.

One may wonder why the discrete logarithm over $(\mathbb{Z}_p,+)$ is easy, but over a general finite cyclic $(\mathcal{G},\circ)$ of order $p$---which is isomorphic to $(\mathbb{Z}_p,+)$---may be hard. The reason is that an isomorphism between $(\mathbb{Z}_p,+)$ and $(\mathcal{G},\circ)$ is established through the correspondence $k\leftrightarrow \underbrace{a\circ a\circ\ldots \circ a}_{k\text{ times}}$, where $k\in\mathbb{Z}_p$ and $a$ is an arbitrary non-identity element of $\mathcal{G}$. But it is widely believed that this isomorphism itself is a one-way function\footnote{Informally, a one-way function is a function that is easy to compute on every input, but hard to invert given the image of a random input.} for sufficiently complex groups (like elliptic ones).

We note that the best classical methods for solving the DLP in general case---Baby-step Giant-step  \citep{daniel1971class} and Pollard's rho  \citep{Pollard1975AMC}---require $O(\sqrt{p})$ computational steps, where $p$ is the group order, i.e. they are exponential in the bitlength of $p$.

\section{Main Result}

Let $(\mathcal{G},\circ)$ be a finite cyclic group of prime order $p$, and $a\in\mathcal{G}\setminus\{1\}$. Consider a function
\begin{equation}
    h_{a}(x):=(-1)^{\log_a x}=\begin{cases}
    -1\quad&\text{if }\log_a x\equiv1\bmod 2,\\
    +1&\text{if }\log_a x\equiv0\bmod 2,
    \end{cases},\quad x\in\mathcal{G}\setminus\{1\}\label{eq:DL_parity_bit}
\end{equation}
which is essentially a parity bit of the discrete logarithm $\log_a x$. Suppose we want to learn $h_a(x)$ using a gradient-based method (e.g., deep learning). For this, consider the stochastic optimization problem 
\begin{equation}
    L_{a}(\mathbf{w}):=\E_{{X}\sim\mathcal{G}\setminus\{1\}}[\ell(f_\mathbf{w}({X}),h_{a}({X}))]\to\min_{\mathbf{w}},\label{loss}
\end{equation}
where $\ell$ is a loss function, ${X}$ are the random inputs (from $\mathcal{G}\setminus\{1\}$), and $f_\mathbf{w}$ is some model parametrized by a parameter vector $\mathbf{w}$ (e.g. a neural network of a certain architecture). We  assume that $L_a(\mathbf{w})$ is differentiable with respect to $\mathbf{w}$. We are interested in studying the \emph{variance} of the gradient of $L_A$  when $A$ is drawn uniformly at random from $\mathcal{G}\setminus\{1\}$. 
The following theorem bounds this variance term.

\begin{theorem} \label{thm:main} 
Suppose that $f_\mathbf{w}(\mathbf{x})$ is differentiable w.r.t. $\mathbf{w}$, and for some scalar $d(\mathbf{w})$, satisfies  
$
\E_{X\sim\mathcal{G}\setminus\{1\}}\left[\left\|\frac{\partial}{\partial\mathbf{w}}f_\mathbf{w}({X})\right\|^2\right]\le d(\mathbf{w})^2
$. Let the loss function $\ell$ in \eqref{loss} be either the square loss $\ell(\hat{y},y)=\frac12(\hat{y}-y)^2$ or a classification loss of the form $\ell(\hat{y},y)=s(\hat{y}\cdot y)$ for some $1$-Lipschitz function $s$. Then 
\begin{equation}
\E_{A\sim\mathcal{G}\setminus\{1\}}\left\|\nabla L_{A}(\mathbf{w})-\boldsymbol{\mu}(\mathbf{w})\right\|^2\le\frac{c\cdot d(\mathbf{w})^2\ln p}{\sqrt{p}},\label{eq:main}
\end{equation}
where $\boldsymbol{\mu}(\mathbf{w}):=\E_{A\sim\mathcal{G}\setminus\{1\}}\nabla L_A(\mathbf{w})$, and $c$ is an absolute constant.
\end{theorem}

\begin{remark}\label{rem:sgd_fails}
Theorem~\ref{thm:main}  says that the gradient of $L_a(\mathbf{w})$ at any point $\mathbf{w}$ is extremely concentrated around a fixed point \textbf{independent} of the base $a$.\footnote{Using our result and Chebyshev’s inequality, one can show that the gradient deviates in 2-norm from a fixed point $\boldsymbol{\mu}(\mathbf{w})$ by more than $\widetilde{\Omega}(2^{-n/6})$ with probability at most $\widetilde{O}(2^{-n/6})$, where $n$ is the bit length of $p$ (group order).} Using this one can show \cite[Theorem 10]{DBLP:journals/jmlr/Shamir18} that a gradient-based method will likely fail in returning a reasonable predictor of the discrete logarithm's parity bit unless the number of iterations is exponentially large in the bitlength $n$ of $p$. This provides strong evidence that gradient-based methods cannot learn even a single bit of the discrete logarithm in $\mathrm{poly}(n)$ time. The result holds \textbf{regardless of which class of predictors we use} (e.g. arbitrarily complex neural networks) --- the problem lies in using gradient-based method to train them.    
\end{remark}

\paragraph{Proof Idea.} Our result is an extension of the work of \cite{DBLP:conf/icml/Shalev-ShwartzS17} which shows that the gradient is not  informative when learning a class of orthogonal functions (see their Theorem~1). In their proof, they rely on the Bessel inequality, which is valid for an orthonormal sequence in the inner-product space. Unfortunately, their result cannot be applied to the parity bit of the discrete logarithm, because the functions  $\{h_a(x)\mid a\in\mathcal{G}\setminus\{1\}\}$, where $h_a(x)$ is defined by \eqref{eq:DL_parity_bit}, are not orthogonal, i.e. $\langle h_a, h_b\rangle:=\E_{X\sim\mathcal{G}\setminus\{1\}}\left[h_a(X)\cdot h_b(X)\right]\ne0$ for some $a,b\in\mathcal{G}\setminus\{1\}$, $a\ne b$. However, we can show that, on average over $a$ and $b$, the inner product $\langle h_a, h_b\rangle$ is small (Lemma~\ref{lem:sum_of_squares}). More precisely, it satisfies
\begin{equation}
\E_{\substack{A,B\sim\mathcal{G}\setminus\{1\}\\A\ne B}}\left[\langle h_A, h_B\rangle^2\right]\ll\frac{\ln^2 p}{p}.\label{eq:inner_prod_small}
\end{equation}
Further, using the Boas-Bellman inequality (Lemma~\ref{lem:boas_bellman}) instead of the Bessel inequality in the proof of \cite{DBLP:conf/icml/Shalev-ShwartzS17}, we can get the bound \eqref{eq:main}.

We note that in order to prove \eqref{eq:inner_prod_small}, we have established some intermediate results that may be of independent interest. Namely, we have shown that for the matrix 
$$
\mathbf{\Phi'}:=\{(-1)^{jk\bmod p}\}_{j,k\in\mathbb{Z}_p^\ast},
$$
the spectral norm $\|\mathbf{\Phi'}\|\ll \sqrt{p}\ln p$ (Lemma~\ref{final}). Using this fact, we proved that for $Y\sim\mathbb{Z}_p^\ast$ a random variable 
\begin{equation}
f(Y)=\E_{X\sim\mathbb{Z}_p^\ast}\left[(-1)^{X}\cdot(-1)^{YX\bmod p}\right],\label{eq:f}
\end{equation}
is concentrated around its mean $\E[f(Y)]=0$ with a variance $\Var[f(Y)]\ll \frac{\ln^2 p}{p}$. And this, in turn, implies \eqref{eq:inner_prod_small}, as is shown in Lemma~\ref{lem:sum_of_squares}.

The proof of Theorem~\ref{thm:main} is given in Section~\ref{sec:proofs}.

\section{Empirical Verification}\label{sec:experim}
Our code is available at \url{https://github.com/armanbolatov/hardness_of_learning}.

\paragraph{Concentration of the Gradient.} As mentioned earlier, Theorem~\ref{thm:main} is true for any finite cyclic group of prime order, including the additive group $(\mathbb{Z}_p, +)$. Let us verify empirically the statement of the theorem for this group. Let $f_\mathbf{w}(x)$ be a neural network\footnote{3-layer dense neural network with 1000 neurons on each hidden layer, sigmoid activation, binary cross-entropy loss.} that we may want to train to learn the mapping\footnote{We remind the reader that for prime $p$, in $(\mathbb{Z}_p,+)$ we have $\log_a x\equiv a^{-1}x\bmod p$, where $a^{-1}$ is the multiplicative inverse of $a$ (see Section~\ref{sec:dlp}).} $x \mapsto (-1)^{\log_a x}$, $x\in\mathbb{Z}_p^\ast$, 
where $\mathbf{w}\in\mathbb{R}^d$ are all parameters of the neural network. Let $\nabla L_a(\mathbf{w})$ be the gradient of the binary cross-entropy loss function at $\mathbf{w}$. We sample $\mathbf{w}_1,\ldots,\mathbf{w}_{20}$ from $\mathbb{R}^d$ using the default PyTorch initializer\footnote{For each dense layer of shape $d_\text{in}\times d_\text{out}$, PyTorch initializes its parameters uniformly at random from the interval $\left[-1/{\sqrt{d_\text{in}}},1/{\sqrt{d_\text{in}}}\right]$, where $d_\text{in}$ is the size of the input, and $d_\text{out}$ is the size of the output.}, and for each $\mathbf{w}_i$, we compute 
\begin{align}
    v(\mathbf{w}_i)&=\E_{A\sim\mathbb{Z}_p^\ast}\left\|\nabla L_A(\mathbf{w}_i)-\E_{A'\sim\mathbb{Z}_p^\ast}\nabla L_{A'}(\mathbf{w}_i)\right\|^2,\\
    g(\mathbf{w}_i)&=\E_{X\sim\mathbb{Z}_p^\ast}\left\|\frac{\partial}{\partial\mathbf{w}}f_\mathbf{w}({X})\right\|^2.\label{eq:g_w}
\end{align}
According to Theorem~\ref{thm:main}, the values $\frac{v(\mathbf{w}_i)}{g(\mathbf{w}_i)}$ should be of order $\widetilde{O}\left(\frac{1}{\sqrt{p}}\right)$. Thus we plot $\E_{i\sim\{1,\ldots,20\}}\left[\frac{v(\mathbf{w}_i)}{g(\mathbf{w}_i)}\cdot\sqrt{p}\right]$ against $p$ in Figure~\ref{fig:thm1_verif}.
\begin{figure}
\begin{minipage}[t]{.48\textwidth}
    \centering
    \includegraphics[width=.9\textwidth]{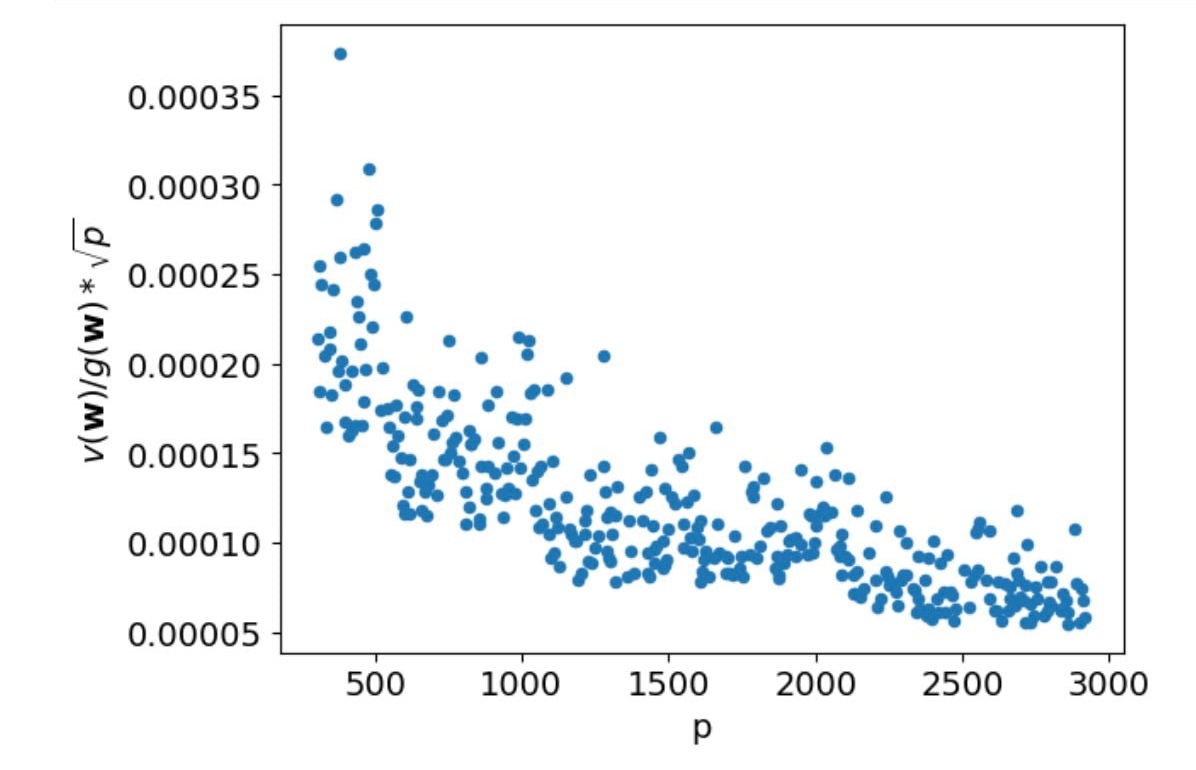}
    \caption{Verifying the statement of Theorem~\ref{thm:main}. For prime numbers $p$ in $[300,3000]$, we plot the left-hand side of \eqref{eq:main} divided by the average squared norm of the neural network's gradient \eqref{eq:g_w} and multiplied by $\sqrt{p}$. The resulting curve is of order $\widetilde{O}(1)$. Moreover, it even decreases.}
    \label{fig:thm1_verif}    
\end{minipage}\hfill\begin{minipage}[t]{.48\textwidth}
    \centering
    \includegraphics[width=.9\textwidth]{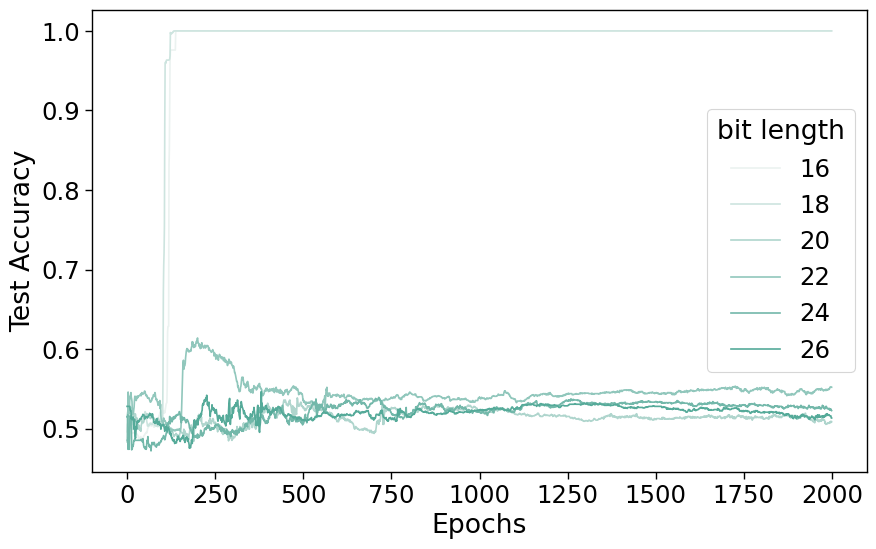}
    \caption{Learning the parity bit of the discrete logarithm in $(\mathbb{Z}_p,+)$ with a 3-layer width-1000 dense network. Darker shades correspond to longer bitlengths. For each bitlength $n$, the group order $p$ is chosen randomly from the prime numbers in the interval $[2^{n-1}, 2^n-1]$.}
    \label{fig:sgd_failure}    
\end{minipage}
\end{figure}
As we can see, this expression is bounded as $p$ grows, which confirms the statement of the theorem. In fact, it is not only bounded but actually decreases, suggesting that our upperbound \eqref{eq:main} can be improved.

\paragraph{Failure of Gradient-Based Learning.} According to Remark~\ref{rem:sgd_fails}, any  gradient-based method most likely will fail to learn the parity bit of a discrete logarithm. To test this claim, we generated a labeled sample 
$$
(x_1, (-1)^{\log_a x_1}), \ldots, (x_m, (-1)^{\log_a x_m})
$$
where $x_1, \ldots, x_m$, and $a$ are taken randomly from $\mathbb{Z}_p^\ast$. Using this sample, we trained a dense 3-layer neural network with 1000 neurons in each hidden layer. We used Adam with a learning rate of 0.001 (default), $m=5000$, a 70/30 split between training and test sets, batch size 100, and we trained for 2000 epochs. The results for different bitlengths $n$ are shown in Figure~\ref{fig:sgd_failure}. The group order $p$ for each bitlength $n$ was taken randomly from the prime numbers in the interval $[2^{n-1}, 2^n-1]$. We can see that as the bitlength increases, the chances of successful learning decrease, as predicted by our theory.

\section{Proofs}\label{sec:proofs}
\subsection{Some Statistical Properties of $\mathbb{Z}_p^\ast$}

The goal of this subsection is to study the distribution of the random variable $f(Y)$, defined by \eqref{eq:f}, where $Y$ is sampled uniformly at random from ${\mathbb Z}_p^\ast$. Our key result is the following theorem that will be crucial for our further analysis of the discrete logarithm's parity bit in Section~\ref{sec:dlp}.
\begin{theorem}\label{major} Let $f(Y)$ be a random variable defined by \eqref{eq:f}, where $Y\sim\mathbb{Z}_p^\ast$. Then
$$
    {\E}[f(Y)]=0,\qquad {\Var}[f(Y)]\leq \frac{c\ln^2 p}{p},
$$
where $c$ is some universal constant.
\end{theorem}

Let us denote $\phi(X,Y)=(-1)^{X}\cdot(-1)^{YX \bmod p}$ and assume that $X$ is also sampled uniformly from ${\mathbb Z}_p^\ast$. Then, $f(Y)={\E}_{X}[\phi(X,Y)]$. Let us first study the distribution of the random variable $\phi(X,Y)$. Our significant finding is the following lemma.
\begin{lemma}\label{lem:mean} ${\E}_{Y}[\phi(x,Y)]=0$ for any $x\in {\mathbb Z}_p^\ast$.
\end{lemma}
\begin{proof}
Direct computation gives
\begin{equation*}
{\E}_{Y}[\phi(x,Y)]=(-1)^{x}\frac{1}{p-1}\sum_{y\in {\mathbb Z}_p^\ast} (-1)^{yx\bmod p}=(-1)^{x}\frac{1}{p-1}\sum_{y'\in {\mathbb Z}_p^\ast} (-1)^{y'}=0,
\end{equation*}
due to the fact that $\{yx\bmod p\mid y\in {\mathbb Z}_p^\ast\}={\mathbb Z}_p^\ast$ and $|{\mathbb Z}_p^\ast|$ is even.
\end{proof}

From Lemma~\ref{lem:mean} we conclude that 
\begin{equation}
{\E}_{Y}[f(Y)] = {\E}_{X,Y}[\phi(X,Y)]=0.\label{eq:mean_f}
\end{equation}
Thus, $f(Y)$ is distributed around its mean $0$ and the first statement of Theorem~\ref{major} is proved. Let us now study its variance. 
The following lemma shows that the spectral norm of the matrix $\mathbf{\Phi}=[\phi(x,y)]_{(x,y)\in ({\mathbb Z}_p^\ast)^2}\in {\mathbb R}^{(p-1)\times (p-1)}$ bounds the variance.
\begin{lemma}\label{spectr} ${\Var}[f(Y)]\leq \frac{\sigma_1(\mathbf{\Phi})^2}{(p-1)^2}$.
\end{lemma}
\begin{proof}
Formula \eqref{eq:mean_f} implies that the variance is equal to the second moment, therefore 
\begin{equation*}
\begin{split}
{\Var}[f(Y)]&= {\mathbb E}[f(Y)^2] = \frac{1}{p-1}\sum_{y\in {\mathbb Z}_p^\ast}\frac{1}{(p-1)^2}\left(\sum_{x\in {\mathbb Z}_p^\ast}\phi(x,y)\right)^2 \\
&=\frac{1}{(p-1)^3}\sum_{(y,x,x')\in ({\mathbb Z}_p^\ast)^3}\phi(x,y)\phi(x',y) = \frac{1}{(p-1)^3}\sum_{(x,x')\in ({\mathbb Z}_p^\ast)^2} (\mathbf\Phi\mathbf\Phi^\top)_{x,x'} \\
&=\frac{1}{(p-1)^2}\sum_{(x,x')\in ({\mathbb Z}_p^\ast)^2} (\mathbf{\Phi\Phi}^\top)_{x,x'}\xi_x\xi_{x'},
\end{split}
\end{equation*}
where $\boldsymbol{\xi}\in {\mathbb R}^{p-1}$ satisfies $\xi_i=\frac{1}{\sqrt{p-1}}, i\in {\mathbb Z}_p^\ast$. The latter quadratic form is bounded by the largest eigenvalue of the symmetric matrix $\mathbf{\Phi\Phi}^\top$, i.e. 
\begin{equation*}
{\rm Var}[f(Y)]= {\E}[f(Y)^2] \leq 
\frac{1}{(p-1)^2}\lambda_1(\mathbf{\Phi\Phi}^\top) = \frac{1}{(p-1)^2}\sigma_1(\mathbf{\Phi})^2. 
\end{equation*}
\end{proof}

Our next goal will be to study the spectral norm of $\mathbf\Phi$. The following lemma simplifies our task.
\begin{lemma}\label{phi_prime} Let $\mathbf{\Phi'}=[(-1)^{yx\bmod p}]_{(x,y)\in ({\mathbb Z}_p^\ast)^2}$. Then we have  $
\sigma_\ell(\mathbf{\Phi})=\sigma_\ell(\mathbf{\Phi'})$, $\ell\in\mathbb{Z}_p^\ast$.
\end{lemma}
\begin{proof} Let $\mathds{1}\in {\mathbb R}^{p-1}$ be a vector with all components equal to 1 and $\boldsymbol{\xi}\in {\mathbb R}^{p-1}$ be such that $\xi_x = (-1)^x$. Since $\mathbf{\Phi}_{x,y} = (-1)^{x}(-1)^{yx\bmod p}$ we have $\mathbf{\Phi} = (\mathds{1}\boldsymbol{\xi}^\top)\odot\mathbf{\Phi'}$ where $\odot$ is the Hadamard product. Let us consider the singular value decomposition (SVD) of $\mathbf{\Phi'}$, i.e. $
\mathbf{\Phi'}=\sum_{\ell=1}^{p-1} \sigma_\ell(\mathbf{\Phi'}){\mathbf u}_\ell{\mathbf v}_\ell^\top$.
Then, we have 
\begin{equation*}
\mathbf{\Phi}=(\mathds{1}\boldsymbol{\xi}^\top)\odot\mathbf{\Phi'}=\sum_{\ell=1}^{p-1} \sigma_i(\mathbf{\Phi'})(\mathds{1}\boldsymbol{\xi}^\top)\odot ({\mathbf u}_\ell{\mathbf v}^\top_\ell)=\sum_{\ell=1}^{p-1} \sigma_\ell(\mathbf{\Phi'})(\mathds{1}\odot{\mathbf u}_\ell)(\boldsymbol{\xi}\odot {\mathbf v}_\ell)^\top.
\end{equation*}
Since $\{\mathds{1}\odot{\mathbf u}_\ell\mid \ell\in [p-1]\}$ and $\{\boldsymbol{\xi}\odot {\mathbf v}_\ell \mid \ell\in [p-1]\}$ are both orthonormal systems of vectors, the latter expression is an SVD of $\mathbf{\Phi}$. Therefore, $\sigma_\ell(\mathbf{\Phi})=\sigma_\ell(\mathbf{\Phi'})$.
\end{proof}

A final result concerning the largest singular value of $\mathbf{\Phi'}$ requires some additional lemmas. Let us denote the vector $[(-1)^x]_{x\in {\mathbb Z}_p}\in {\mathbb R}^{p}$ by ${\mathbf a}$. Let $\omega=e^{\frac{2\pi {i}}{p}}$ be a primitive $p$th root of unity. Other primitive roots of unity are $\omega_2, \cdots, \omega_{p-1}$ where $\omega_\ell = \omega^\ell, \ell\geq 0$. Let $\mathbf{U}_\ell = [\omega^{jk}_\ell]_{j,k\in {\mathbb Z}_p}$, then the matrix
$
\frac{1}{\sqrt{p}}\mathbf{U}_\ell
$
is unitary for $\ell\in {\mathbb Z}_p^\ast$. In fact, $\frac{1}{\sqrt{p}}\mathbf{U}_1$ is a discrete Fourier transform (DFT) matrix. From unitarity, we obtain $\|\mathbf{U}_\ell\|=\sqrt{p}$. Let us denote $\mathbf{U}_1=\begin{bmatrix}
{\mathbf b}_0, \cdots, {\mathbf b}_{p-1}
\end{bmatrix}$.

\begin{lemma}\label{lem:a} The vector $\mathbf{a}$ can be decomposed as  
$
{\mathbf a} = \frac{2}{p}\sum_{\ell=0}^{p-1}\frac{1}{1+\omega^{-\ell}}{\mathbf b}_\ell
$.
\end{lemma} 
\begin{proof} From unitarity, we conclude that $\left\{{\mathbf e}_\ell=\frac{1}{\sqrt{p}}{\mathbf b}_\ell\right\}_{\ell=0}^{p-1}$ is an orthonormal basis in ${\mathbb C}^{p}$. Therefore,
$
{\mathbf a} = \sum_{\ell=0}^{p-1}({\mathbf e}_\ell^\dag {\mathbf a}){\mathbf e}_\ell$. After computation
\begin{equation*}
({\mathbf e}_\ell^\dag {\mathbf a})=\frac{1}{\sqrt{p}}\sum_{x=0}^{p-1}\omega^{-x\ell}(-1)^x = \frac{1}{\sqrt{p}}\sum_{x=0}^{p-1} (-\omega^{-\ell})^x= \frac{1}{\sqrt{p}}\cdot\frac{1-(-\omega^{-\ell})^{p}}{1+\omega^{-\ell}}
=\frac{2}{\sqrt{p}}\cdot\frac{1}{1+\omega^{-\ell}},
\end{equation*}
we conclude that 
\begin{equation*}
\begin{split}
{\mathbf a} = \sum_{\ell=0}^{p-1}\frac{2}{\sqrt{p}}\cdot\frac{1}{1+\omega^{-\ell}}{\mathbf e}_\ell = \sum_{\ell=0}^{p-1}\frac{2}{p}\cdot\frac{1}{1+\omega^{-\ell}}{\mathbf b}_\ell.
\end{split}
\end{equation*}
\end{proof}
\begin{corollary}\label{corol} The matrix $\mathbf\Phi'$ can be represented as
$
\mathbf{\Phi'} = \sum_{\ell=0}^{p-1}\frac{2}{p}\cdot\frac{1}{1+\omega^{-\ell}}\mathbf{\Omega}_\ell
$, where $\mathbf{\Omega}_\ell$ is a submatrix of $\mathbf{U}_\ell$ obtained after deletion of the first row and the first column. 
\end{corollary} 
\begin{proof} Unfolded component-wise, the corollary is equivalent to 
\begin{equation*}
\begin{split}
(-1)^{jk\,{\rm mod\,\,}p} = \sum_{\ell=0}^{p-1}\frac{2}{p}\cdot\frac{1}{1+\omega^{-\ell}}\omega_\ell^{jk},
\end{split}
\end{equation*}
for any $j,k\in {\mathbb Z}_p^\ast$.
After setting $s=jk\,{\rm mod\,\,}p$ and using Lemma~\ref{lem:a} we conclude
\begin{equation*}
\begin{split}
(-1)^{jk\,{\rm mod\,\,}p} &= ({\mathbf a})_s=\sum_{\ell=0}^{p-1}\frac{2}{p}\cdot\frac{1}{1+\omega^{-\ell}}({\mathbf b}_\ell)_s = \sum_{\ell=0}^{p-1}\frac{2}{p}\cdot\frac{1}{1+\omega^{-\ell}}\omega^{\ell s} \\
&=\sum_{\ell=0}^{p-1}\frac{2}{p}\cdot\frac{1}{1+\omega^{-\ell}}\omega^{\ell (jk\bmod p)} = \sum_{\ell=0}^{p-1}\frac{2}{p}\cdot\frac{1}{1+\omega^{-\ell}}\omega^{jk}_\ell,
\end{split}
\end{equation*}
which concludes the proof.
\end{proof}

Now it remains to bound $\sigma_1(\mathbf{\Phi'})$.
\begin{lemma}\label{final} $\sigma_1(\mathbf{\Phi'})\leq c_1p^{1/2}\ln p$ where $c_1$ is some universal constant.
\end{lemma} 
\begin{proof}
Using Corollary~\ref{corol} we conlcude
\begin{equation*}
\begin{split}
\|\mathbf{\Phi'}\|=\left\|\sum_{\ell=0}^{p-1}\frac{2}{p}\cdot\frac{1}{1+\omega^{-\ell}}\mathbf{\Omega}_\ell\right\|\leq \sum_{\ell=0}^{p-1}\frac{2}{p}\cdot\frac{1}{|1+\omega^{-\ell}|}\|\mathbf{\Omega}_\ell\|.
\end{split}
\end{equation*}
The spectral norm of every submatrix is not greater than that of the matrix, therefore, $\|\mathbf{\Omega}_\ell\|\leq \|\mathbf{U}_\ell\|=\sqrt{p}$. Thus, 
\begin{equation*}
\begin{split}
\|\mathbf{\Phi'}\|\leq \sqrt{p}\sum_{\ell=0}^{p-1}\frac{2}{p}\frac{1}{|1+\omega^{-\ell}|}.
\end{split}
\end{equation*}
Now it remains to bound the sum $\sum_{\ell=0}^{p-1}\frac{2}{p}\frac{1}{|1+\omega^{-\ell}|}=\sum_{\ell=0}^{p-1}\frac{2}{p}\frac{1}{|1+\omega^{\ell}|}$. Let $\theta=\frac{2\pi \ell}{p}$. Note  that $\theta\in [0,2\pi)$ and
\begin{equation*}
\begin{split}
|1+\omega^{\ell}| = |1+e^{{i}\theta}| = (2+2\cos(\theta))^{1/2}=2\left|\cos\left(\frac{\theta}{2}\right)\right|.
\end{split}
\end{equation*}
Let us denote $2\psi=\theta-\pi$. Thus, $|1+\omega^{\ell}| =2\left|\cos\left(\frac{2\psi+\pi}{2}\right)\right| = 2|\sin(\psi)|\geq |\psi|$ if $\psi\in \left[-\frac{\pi}{4}, \frac{\pi}{4}\right]$. Note that $\psi\in \left[-\frac{\pi}{4}, \frac{\pi}{4}\right]$ if and only if $-\frac{\pi}{4}\leq \frac{\pi k}{p}-\frac{\pi}{2}\leq \frac{\pi}{4}$, or $\frac{1}{4}p\leq k\leq \frac{3}{4}p$. Thus, we have 
\begin{equation*}
\begin{split}
\sum_{\ell\in \left[\frac{1}{4}p, \frac{3}{4}p\right]\cap {\mathbb Z}_p}\frac{2}{p}\cdot\frac{1}{|1+\omega^{\ell}|}&\leq \sum_{\ell\in \left[\frac{1}{4}p, \frac{3}{4}p\right]\cap {\mathbb Z}_p}\frac{2}{p}\cdot\frac{1}{\left|\frac{\pi \ell}{p}-\frac{\pi}{2}\right|}=\frac{4}{\pi}\sum_{\ell\in \left[\frac{1}{4}p, \frac{3}{4}p\right]\cap {\mathbb Z}_p}\frac{1}{|2\ell-p|} \\
&\leq\frac{8}{\pi}\sum_{j=1}^{\lceil p/2 \rceil}\frac{1}{j} \ll \ln p.
\end{split}
\end{equation*}
Since $\left|1+\omega^{\ell}\right| = 2|\sin(\psi)|\geq 1$ if $\psi\in \left[-\frac{\pi}{2}-\frac{\pi}{4}\right]\cup \left[\frac{\pi}{4}, \frac{\pi}{2}\right]$, then
\begin{equation*}
\begin{split}
\sum_{\ell\in {\mathbb Z}_p:\,\,\psi\in\left[-\frac{\pi}{2}-\frac{\pi}{4}\right]\cup \left[\frac{\pi}{4}, \frac{\pi}{2}\right]}\frac{2}{p}\cdot\frac{1}{\left|1+\omega^{\ell}\right|}\asymp 1.
\end{split}
\end{equation*}
Thus, the total sum satisfies
\begin{equation*}
\begin{split}
\sum_{\ell\in {\mathbb Z}_p}\frac{2}{p}\frac{1}{\left|1+\omega^{\ell}\right|}\ll \ln p.
\end{split}
\end{equation*}
\end{proof}
\begin{proof}\textbf{of Theorem~\ref{major}}.  
Using Lemma~\ref{spectr} we have ${\Var}[f(Y)]\leq \frac{\sigma_1(\mathbf{\Phi})^2}{(p-1)^2}$. Lemma~\ref{phi_prime} additionally gives us ${\Var}[f(Y)]\leq \frac{\sigma_1(\mathbf{\Phi'})^2}{(p-1)^2}$.  Finally, using Lemma~\ref{final} we conlude that ${\Var}[f(Y)]\ll \frac{\ln^2 p}{p}$.
\end{proof}

\subsection{Near Orthogonality of Discrete Logarithm's Parity Bits}\label{sec:dlp}

As mentioned in Section 3, the main tool for adapting the proof of \cite{DBLP:conf/icml/Shalev-ShwartzS17} to our needs is the Boas-Bellman inequality, which we present below.

\begin{lemma}[Boas-Bellman inequality] \label{lem:boas_bellman} Let $h_1,\ldots,h_{m},g$ be elements of an inner product space. Then
$$
\sum_{i=1}^{m}\langle h_i, g\rangle^2\le\|g\|^2\left(\max_i\|h_i\|^2+\sqrt{\sum_{i\ne j}\langle h_i,h_j\rangle^2}\right).
$$
\end{lemma}
\begin{proof} Can be found in the works of \cite{boas1941general} and \cite{bellman1944almost}.
\end{proof}

As we can see, this inequality turns into Bessel's inequality for an orthonormal sequence $\{h_i\}$. In our case, the functions $h_a(x)$ given by \eqref{eq:DL_parity_bit} are not pairwise orthogonal. However, it can be shown that they are approximately pairwise orthogonal. This is what we will do in this subsection. First, we derive the distribution of $\log_A B$ when $A$ and $B$ are sampled uniformly at random from $\mathcal{G}\setminus\{1\}$.

\begin{lemma}\label{lem:unif}
    Let $(\mathcal{G},\circ)$ be a finite cyclic group of prime order $p>2$. Let $A,B\sim\mathcal{G}\setminus\{1\}$. Then the distribution of $\log_B A$ is uniform over $\{1,\ldots,p-1\}$.
\end{lemma}
\begin{proof}
    For any $y\in\{1,\ldots,p-1\}$, we have
    \begin{align*}
        &\Pr_{A,B\sim\mathcal{G}\setminus\{1\}}\left[\log_B A=y\right]=\E_{A,B\sim\mathcal{G}\setminus\{1\}}\left[\mathbb{I}(\log_B A=y)\right]\\
        &=\frac1{(p-1)^2}\sum_{a,b\in\mathcal{G}\setminus\{1\}}\mathbb{I}[\log_b a=y]=\frac{1}{(p-1)^2}\sum_{b\in\mathcal{G}\setminus\{1\}}\sum_{a\in\mathcal{G}\setminus\{1\}}\underbrace{\mathbb{I}[\log_b a=y]}_{=1\,\,\text{iff}\,\,a=b^y}\\
        &=\frac{1}{(p-1)^2}\sum_{b\in\mathcal{G}\setminus\{1\}}1=\frac{1}{p-1}.
    \end{align*}
\end{proof}

Now we are ready to prove the advertised bound \eqref{eq:inner_prod_small}.

\begin{lemma} \label{lem:sum_of_squares}     Let $(\mathcal{G},\circ)$ be a finite cyclic group of prime order $p>2$. Let $a\in\mathcal{G}\setminus\{1\}$. Consider a function $h_a(x)$ defined in \eqref{eq:DL_parity_bit}. Then
\begin{equation*}
\sum_{\substack{a,b\in\mathcal{G}\setminus\{1\}\\a\ne b}}{\mathbb E}[h_a(X)\cdot h_b(X)]^2\leq c_2 p\ln^2 p
\end{equation*}
for some universal constant $c_2$.
\end{lemma}
\begin{proof}
\begin{equation*}
\begin{split}
&\frac{1}{(p-1)^2}\sum_{a\in \mathcal{G}\setminus\{1\}}\sum_{b\in \mathcal{G}\setminus\{1\}}{\mathbb E}[h_a(X)\cdot h_b(X)]^2\\
&=\E_{A,B\sim \mathcal{G}\setminus\{1\}}\left[\E_{X\sim \mathcal{G}\setminus\{1\}}[h_A(X)\cdot h_B(X)]^2\right] \\
&=\E_{A,B\sim \mathcal{G}\setminus\{1\}}\left[\E_{X\sim \mathcal{G}\setminus\{1\}}\left[(-1)^{\log_A X}\cdot(-1)^{\log_B X}\right]^2\right]\\
&=\E_{A,B\sim \mathcal{G}\setminus\{1\}}\left[\E_{X\sim \mathcal{G}\setminus\{1\}}\left[(-1)^{\log_A X}\cdot(-1)^{\log_B A\cdot\log_A X\bmod p}\right]^2\right]\\
&=\E_{A,B\sim \mathcal{G}\setminus\{1\}}\left[\E_{X\sim \mathbb{Z}_p^\ast}\left[(-1)^{X}\cdot(-1)^{\log_B A\cdot X\bmod p}\right]^2\right]\\
&{\stackrel{\text{Lem.~\ref{lem:unif}}}{=}}\E_{Y\sim {\mathbb Z}_p^\ast}\left[\E_{X\sim \mathbb{Z}_p^\ast}\left[(-1)^{X}\cdot(-1)^{Y\cdot X\bmod p}\right]^2\right]=\Var_{Y\sim\mathbb{Z}_p^\ast}[f(Y)]\,\,{\stackrel{\text{Thm.~\ref{major}}}{\ll}}\,\,\frac{\ln^2 p}{p}
\end{split}
\end{equation*}
Therefore, $\sum_{a\in \mathcal{G}\setminus\{1\}}\sum_{b\in \mathcal{G}\setminus\{1\}}{\E}[h_a(X)\cdot h_b(X)]^2\ll p\ln^2 p$ and 
$$
\sum_{a\ne b}{\mathbb E}[h_a(X)\cdot h_b(X)]^2\leq c_3 p\ln^2 p-(p-1)\ll p\ln^2 p.
$$
\end{proof}

\subsection{Proof of Theorem~\ref{thm:main}}\label{sec:proof_main}
\begin{proof}
We prove the result for the squared loss $\ell(\hat{y},y)=\frac12(\hat{y}-{y})^2$. The classification loss is handled analogously. Define the vector-valued function 
$$
\mathbf{g}(x)=\frac{\partial}{\partial\mathbf{w}}f_\mathbf{w}(x),
$$
and let $\mathbf{g}(x)=(g_1({x}),g_2({x}),\ldots,g_d({x}))$ for real-valued functions $g_1,\ldots,g_d$. Then we have
\begin{align*}
    \boldsymbol{\mu}(\mathbf{w})&:=\E_{A\sim\mathcal{G}\setminus\{1\}}\nabla L_A(\mathbf{w})=\E_{A\sim\mathcal{G}\setminus\{1\}}\E_{X\sim\mathcal{G}\setminus\{1\}}[(f_\mathbf{w}(X)-h_A(X))\mathbf{g}(X)]\\
    &=\E_{X\sim\mathcal{G}\setminus\{1\}}[f_\mathbf{w}(X)\mathbf{g}(X)]-\E_{X\sim\mathcal{G}\setminus\{1\}}\left\{\mathbf{g}(X)\E_{A\sim\mathcal{G}\setminus\{1\}}[(-1)^{\log_A X}]\right\}\\
    &=\E_{X\sim\mathcal{G}\setminus\{1\}}[f_\mathbf{w}(X)\mathbf{g}(X)]-\E_{X\sim\mathcal{G}\setminus\{1\}}\mathbf{g}(X)\cdot\underbrace{\E_{Y\sim\mathbb{Z}^\ast_p}[(-1)^{Y}]}_{0}=\E_{X\sim\mathcal{G}\setminus\{1\}}[f_\mathbf{w}(X)\mathbf{g}(X)]
\end{align*}
Thus,
\begin{align}
&\E_{A\sim\mathcal{G}\setminus\{1\}}\left\|\nabla L_A(\mathbf{w})-\boldsymbol{\mu}(\mathbf{w})\right\|^2\notag\\
&=\E_{A\sim\mathcal{G}\setminus\{1\}}\left\|\E_{{X}\sim\mathcal{G}\setminus\{1\}}[(f_\mathbf{w}({X})-h_A({X}))\mathbf{g}({X})]-\E_{X\sim\mathcal{G}\setminus\{1\}}[f_\mathbf{w}({X})\mathbf{g}({X})]\right\|^2\notag\\
&=\E_{A\sim\mathcal{G}\setminus\{1\}}\left\|\E_{X\sim\mathcal{G}\setminus\{1\}}[h_{A}({X})\mathbf{g}({X})]\right\|^2=\E_{A\sim\mathcal{G}\setminus\{1\}}\sum_{j=1}^d\left(\E_{{X}\sim\mathcal{G}\setminus\{1\}}[h_{A}({X})g_j({X})]\right)^2\notag\\
&=\E_{A\sim\mathcal{G}\setminus\{1\}}\sum_{j=1}^d\langle h_{A}, g_j\rangle^2.\label{eq:sum_coord}
\end{align}
From Lemmas~\ref{lem:boas_bellman}~and~\ref{lem:sum_of_squares}, we have
\begin{align}
    &\E_{A\sim\mathcal{G}\setminus\{1\}}\langle h_A,g\rangle^2=\frac1{p-1}\sum_{a\in\mathcal{G}\setminus\{1\}}\langle h_a,g\rangle^2\notag\\
    &\le\frac{\|g\|^2}{p-1}\left(\max_{a\in\mathcal{G}\setminus\{1\}}\|h_a\|^2+\sqrt{\sum_{\substack{a,b\in\mathcal{G}\setminus\{1\}\\a\ne b}}\langle h_a, h_b\rangle^2}\right)\le\|g\|^2\left(\frac{1}{p-1}+\frac{\sqrt{c p}\ln p}{p-1}\right)\label{eq:mean_coord}
\end{align}
From \eqref{eq:sum_coord} and \eqref{eq:mean_coord}, we have
\begin{align}
    &\E_{A\sim\mathcal{G}\setminus\{1\}}\left\|\nabla L_A(\mathbf{w})-\E_{{X}\sim\mathcal{G}\setminus\{1\}}[f_\mathbf{w}({X})\mathbf{g}({X})]\right\|^2\le\left(\frac{1}{p-1}+\frac{\sqrt{c p}\ln p}{p-1}\right)\sum_{j=1}^d\|g_j\|^2\notag\\
    &=\left(\frac{1}{p-1}+\frac{\sqrt{c p}\ln p}{p-1}\right)\E_{X\sim\mathcal{G}\setminus\{1\}}\|\mathbf{g}(X)\|^2\le\left(\frac{1}{p-1}+\frac{\sqrt{c p}\ln p}{p-1}\right)d(\mathbf{w})^2\notag\\
    &\ll\frac{\ln p}{\sqrt{p}}d(\mathbf{w})^2.\notag
\end{align}
The proof for the classification loss $\ell(\hat{y},y)=s(\hat{y}\cdot y)$ can be reduced to the proof for the squared loss as it is done in  Theorem~1 of \citet{DBLP:conf/icml/Shalev-ShwartzS17}.
\end{proof}

\section{Additional Experiments}
Here we present the results of experiments that extend the scope of the paper. Namely, we are empirically investigating the learnability of the discrete logarithm \emph{itself} and of \emph{all} its bits, not just one bit. 
\paragraph{Low correlation of discrete logarithms.} We computed the mean squared  covariance 
\begin{equation}
    \E_{A,B\sim\mathbb{Z}_p^\ast}\left(\Cov_{X\sim\mathbb{Z}_p^\ast}[\log_A X,\log_B X]\right)^2\label{eq:mscov}
\end{equation}
for prime numbers in the interval $[3, 500]$. The results are shown in Figure~\ref{fig:cov}.
\begin{figure}
    \centering
    \includegraphics[width=.5\textwidth]{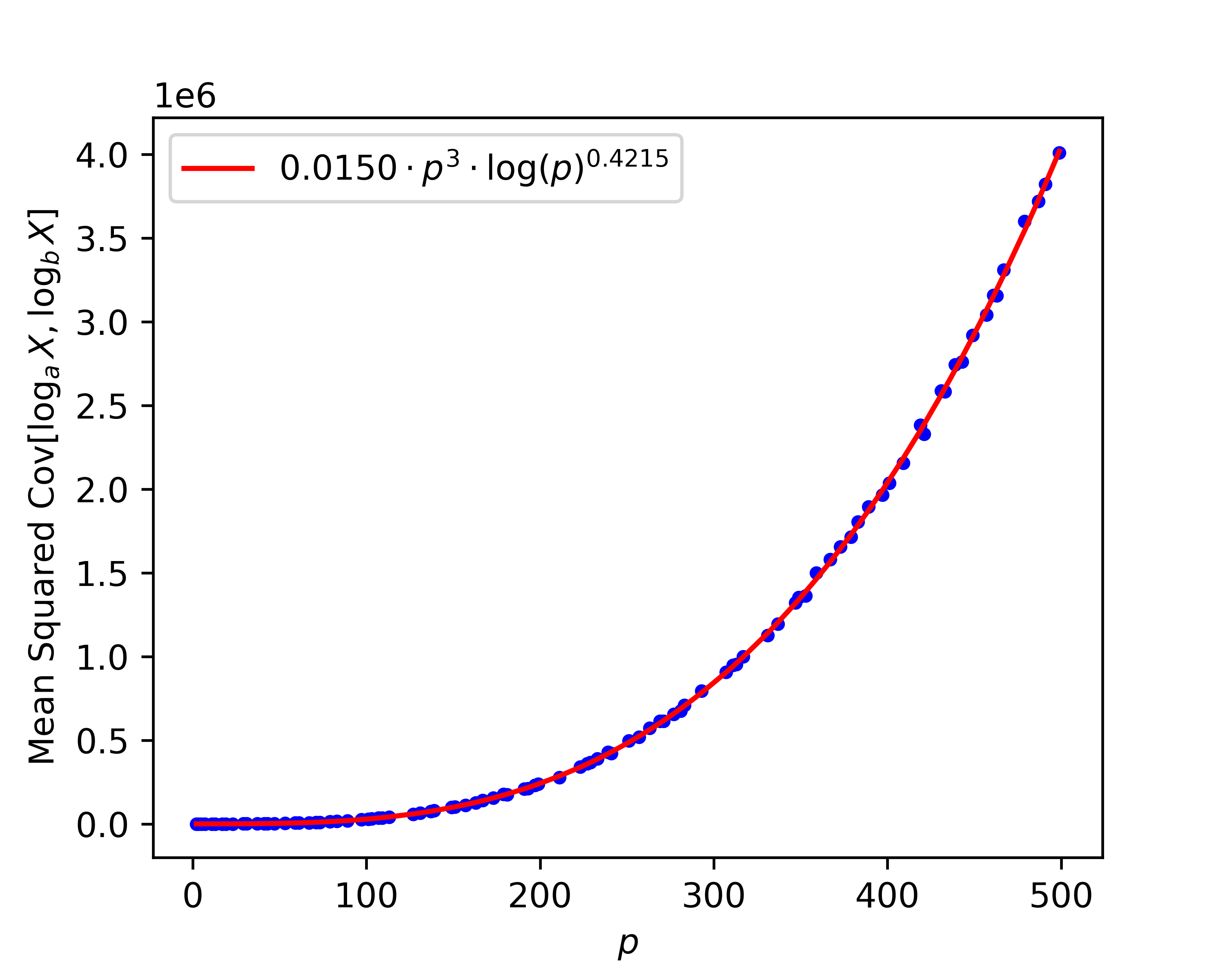}
    \caption{Mean squared covariance between two logarithms, $\log_a X$ and $\log_b X$, when $X$ is a random variable uniformly distributed on $\mathbb{Z}_p^\ast$.}
    \label{fig:cov}
\end{figure}
As we can see, the expression \eqref{eq:mscov} fits the curve $p\mapsto 0.015\cdot p^3\cdot(\ln p)^{0.42}$ well. This suggests that $\Cov_{X\sim\mathbb{Z}_p^\ast}[\log_a X,\log_b X]=\widetilde{O}(p^{3/2})$ on average over $a,b\in\mathbb{Z}_p^\ast$. Since the variance of the discrete logarithm is
$$
\Var_{X\sim\mathbb{Z}_p^\ast}[\log_a X]=\frac{1}{p-1}\sum_{k=1}^{p-1}k^2-\left(\frac{1}{p-1}\sum_{k=1}^{p-1}k\right)^2=\frac{p^2}{12}-\frac{p}{6}=O(p^2)
$$
we can conjecture that the average correlation is
\begin{equation}
\Corr_{X\sim\mathbb{Z}_p^\ast}[\log_a X,\log_b X]=\widetilde{O}\left(\frac{1}{\sqrt{p}}\right).\label{eq:cor_bound}
\end{equation}
Thus, using this estimate in the Boas-Bellman inequality for the class of ``standardized'' discrete logarithms $\{f_a(x)\mid a\in\mathbb{Z}_p^\ast\}$, where 
$$
f_a(x)=\frac{\log_a x-\frac{p}2}{\sqrt{\frac{p^2}{12}-\frac{p}{6}}},\qquad x\in\mathbb{Z}_p^\ast,
$$
we can show that this class is also hard to learn by gradient-based methods. The only thing missing is a rigorous proof of the bound \eqref{eq:cor_bound}. We leave it to our future work.

\paragraph{Failure to learn all bits of the discrete logarithm.} Here we follow the experimental setup from Section~\ref{sec:experim} with the difference that the output of the neural network is not only the parity bit, but all the bits of the discrete logarithm. As a loss function, we use the sum of the cross-entropies for each bit. The results for two different bit lengths are shown in Figure~\ref{fig:all_bits}.
\begin{figure}
    \centering
    \,\qquad\includegraphics[width=.45\textwidth, valign=t]{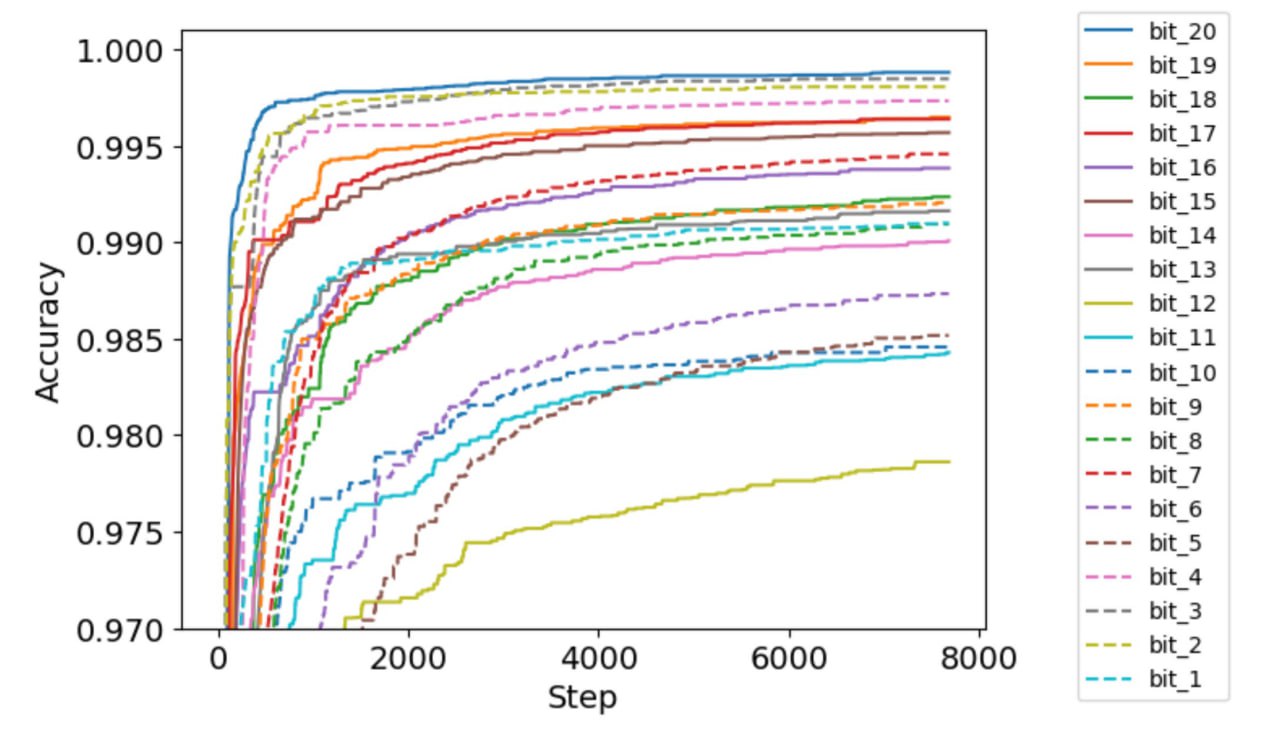}\hfill\includegraphics[width=.4\textwidth, valign=t]{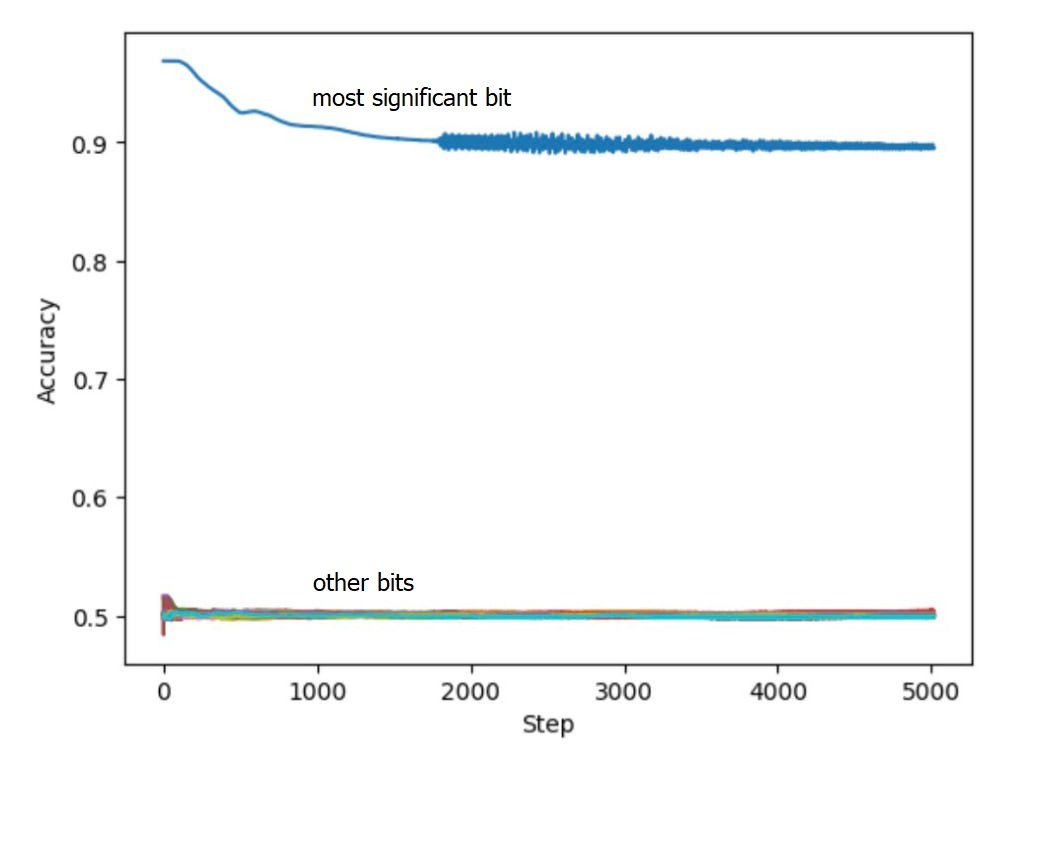}\qquad\,
    \caption{Test Accuracies when learning \emph{all} bits of the discrete logarithm in $(\mathbb{Z}_p,+)$ with a single neural network. Bitlengths of $p$: 20 (left) and 40 (right).}
    \label{fig:all_bits}
\end{figure}
As in the case of one bit, we see that for a longer bit length, the gradient method is not able to learn all the bits of the discrete logarithm. Note that in both cases, the more significant bits are learned better than the less significant ones. We leave the study of this phenomenon to our future work.

\acks{This research has been funded by Nazarbayev University under Faculty-development competitive research grants program for 2023-2025 Grant \#20122022FD4131, PI R. Takhanov.}

\bibliography{ref}

\end{document}